\documentclass{article}
\pdfoutput=1
\usepackage[utf8]{inputenc}
\usepackage{amsmath}
\usepackage{amssymb}
\usepackage{multirow, hhline}
\usepackage[table]{xcolor}
\usepackage[para,online,flushleft]{threeparttable}
\usepackage{microtype}   
\usepackage[letterpaper,margin=1.0in]{geometry}
\usepackage{graphicx}

\usepackage{hyperref}       % hyperlinks
\usepackage{url}            % simple URL typesetting
\usepackage{booktabs}       % professional-quality tables
\usepackage{amsfonts}       % blackboard math symbols
\usepackage{bbding}
\usepackage{xcolor}         % colors
\usepackage[numbers,sort]{natbib}
\usepackage{color}
\usepackage{subcaption}
\usepackage{graphicx}
\usepackage{tablefootnote}
\usepackage{algorithm}
\usepackage{algorithmic}
\usepackage{bbding}
\usepackage{footnote}
\makesavenoteenv{tabular}
\usepackage{threeparttable}
\usepackage{enumitem}
\usepackage{multirow}
\usepackage{multicol}
\usepackage{tabularx}
\usepackage{diagbox}
\usepackage{amsmath}\allowdisplaybreaks
\usepackage{amsfonts,bm}
\usepackage{amssymb}

\def\eqref#1{equation~(\ref{#1})}
\def\1{\bf{1}}

\usepackage{amsthm}
\usepackage{etoolbox}
\theoremstyle{plain}

\AfterEndEnvironment{thm}{\noindent\ignorespaces}

\AfterEndEnvironment{defn}{\noindent\ignorespaces}

\AfterEndEnvironment{exmp}{\noindent\ignorespaces}

\AfterEndEnvironment{lem}{\noindent\ignorespaces}

\AfterEndEnvironment{asm}{\noindent\ignorespaces}
\newtheorem{remark}{Remark}
\AfterEndEnvironment{remark}{\noindent\ignorespaces}

\AfterEndEnvironment{cor}{\noindent\ignorespaces}
\newtheorem{prop}{Proposition}
\AfterEndEnvironment{prop}{\noindent\ignorespaces}

\makeatletter
\def\Ddots{\mathinner{\mkern1mu\raise\p@
\vbox{\kern7\p@\hbox{.}}\mkern2mu
\raise4\p@\hbox{.}\mkern2mu\raise7\p@\hbox{.}\mkern1mu}}
\makeatother

\makeatletter
\newcommand*{\rom}[1]{\expandafter\@slowromancap\romannumeral #1@}
\makeatother

\title{Robust Network Learning via Inverse Scale Variational Sparsification}

\begin{document}

\author{%
  Zhiling Zhou$^{*\diamond}$
  \qquad 
  Zirui Liu$^{*\diamond}$
  \qquad 
  Chengming Xu$^{\diamond}$
  \qquad 
  Yanwei Fu$^{\diamond}$
  \qquad 
  Xinwei Sun$^{\diamond\star}$
}

\makeatletter\def\Hy@Warning#1{}\makeatother
\def\thefootnote{$*$}\footnotetext{Equal Contribution}
\makeatletter\def\Hy@Warning#1{}\makeatother
\def\thefootnote{$\diamond$}\footnotetext{School of Data Science, Fudan University; \{zlzhou20, zrliu20, cmxu18, yanweifu, sunxinwei\}@fudan.edu.cn}
\makeatletter\def\Hy@Warning#1{}\makeatother
\def\thefootnote{$\star$}\footnotetext{Correspondence Author}

\maketitle

\begin{abstract}

While neural networks have made significant strides in many AI tasks, they remain vulnerable to a range of noise types, including natural corruptions, adversarial noise, and low-resolution artifacts. Many existing approaches focus on enhancing robustness against specific noise types, limiting their adaptability to others. Previous studies have addressed general robustness by adopting a spectral perspective, which tends to blur crucial features like texture and object contours. Our proposed solution, however, introduces an inverse scale variational sparsification framework within a time-continuous \emph{inverse scale space} formulation. This framework progressively learns finer-scale features by discerning variational differences between pixels, ultimately preserving only large-scale features in the smoothed image. Unlike frequency-based methods, our approach not only removes noise by smoothing small-scale features where corruptions often occur but also retains high-contrast details such as textures and object contours. Moreover, our framework offers simplicity and efficiency in implementation. By integrating this algorithm into neural network training, we guide the model to prioritize learning large-scale features. We show the efficacy of our approach through enhanced robustness against various noise types.

\end{abstract}

\section{Introduction}

Despite the significant achievements of deep learning models in various imaging tasks \citep{xie2012image, ren2015faster, dosovitskiy2020image, long2015fully}, they are vulnerable to different types of noise, including adversarial noise \citep{goodfellow2014explaining, ilyas2019adversarial}, natural corruptions \citep{hendrycks2018benchmarking}, and compression artifacts in low-resolution images \citep{bulat2018learn, lugmayr2019unsupervised}. This vulnerability can lead to significant safety issues in practical applications, posing a major barrier to model deployment.

Many studies have aimed to enhance robustness against each individual type of noise~\citep{hendrycks2019augmix, lopes2019improving, bulat2018learn, sun2020test, mummadi2019defending, kumari2019harnessing, szegedy2016rethinking}. Most of these rely on data augmentation, using methods tailored to specific types of noise. For example, some works~\citep{hendrycks2019augmix, lopes2019improving, sun2020test} generated images with natural corruptions for training or self-supervised learning. Similarly, adversarial training~\citep{mummadi2019defending, kumari2019harnessing} involves adding small, often imperceptible, perturbations to create adversarial samples for training. Unfortunately, in real-world situations, we can't predict the types of noise we will face, so focusing on just one type is impractical. This limits the usefulness of robust networks. Therefore, we need to find a way to improve network robustness that works for all types of noise in a unified manner.

Several efforts have achieved general robustness by either learning domain-invariant representations for out-of-distribution generalization \citep{li2018deep, sun2021recovering} or by smoothing out high-frequency components \citep{yucel2023hybridaugment++, sun2022spectral}. 
However, learning invariant representations requires data from multiple domains,  which is often hard to obtain. Conversely, while frequency-based methods reduce high-frequency noise, they often blur important features like edges and textures, limiting their effectiveness.

On the other hand, some recent pilot studies aim to improve general robustness by focusing on learning large-scale information (or 'features'). For example, they emphasize larger gradient components of low-dimensional manifolds in natural images \citep{li2020defense} and higher entropy of feature activations \citep{wang2021rethinking}. Theoretically, we will align with such an idea, and re-introduce the classical inverse scale space theory \citep{scherzer2001inverse}, which is a generalized form of Tikhonov-Morozov regularization, can iteratively refine a sequence of inverse scale-space variations of natural images, gradually evolving toward the noisy ones\footnote{This idea essentially resembles the diffusion model.}. This method effectively removes small-scale information, like subtle intensity variations and fine-grained patterns \citep{burger2005nonlinear}, and captures large-scale features, like general shapes, as shown in Fig.~\ref{fig.sol_path} (a).

Formally, we generalize the inverse scale space method, and propose a novel inverse scale variational sparsification method that  can effectively smooth out noise components without blurring important features encoded in the image.
To enforce variational sparsity, we leverage the Total-Variation (TV) regularization \citep{rudin1992nonlinear, chambolle2004algorithm} in a newly proposed ordinary differential equation in the inverse scale space \citep{burger2005nonlinear, scherzer2001inverse, ORXYY16}. Starting from a blank image without any information, this equation will generate a TV-regularized image path, where large-scale features are learned faster than small-scale ones. Therefore, with a proper early-stopping time, it can effectively eliminate small-scale features while preserving large-scale ones in the resulting regularized image, as illustrated in Fig.~\ref{fig.sol_path} (b). Furthermore, equipped with a simple discretization and an efficient sparse projection method, our dynamics can be easily implemented with an iterative algorithm, dubbed as \textbf{Vi}sion \textbf{Ro}bust \textbf{L}inearized \textbf{B}regman \textbf{I}teration (ViRoLBI) in this paper.

\begin{figure*}
    \centering
\includegraphics[width=1.0\textwidth]{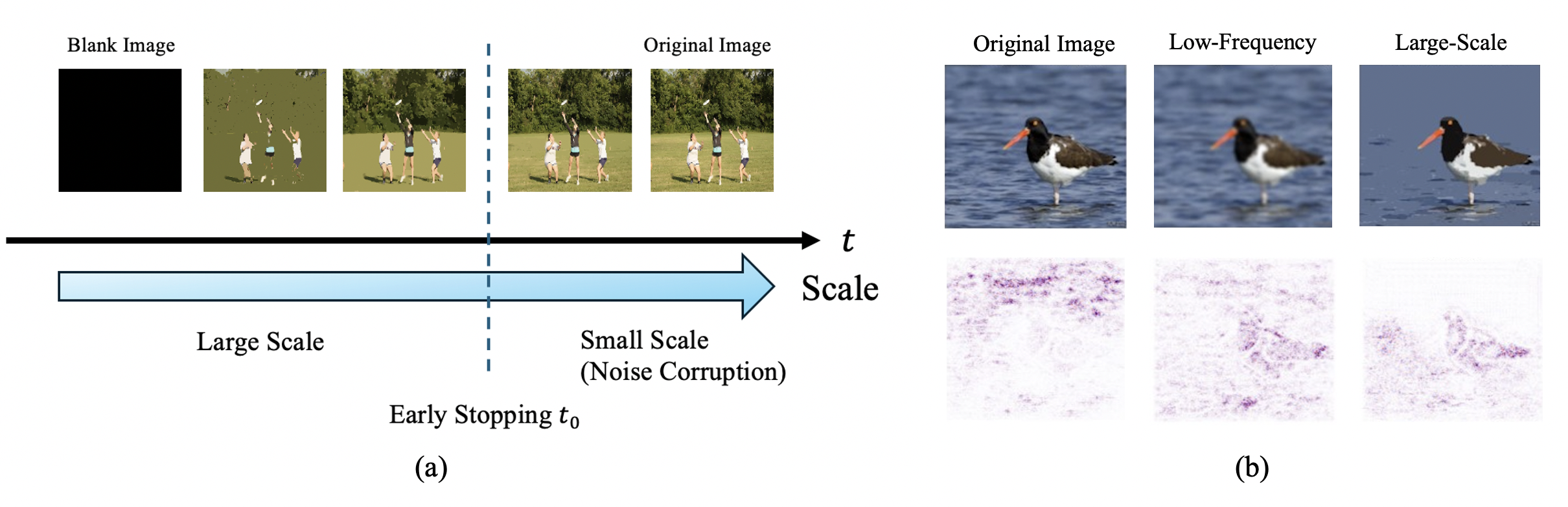}
    \caption{\textbf{(a)} Visualization of Inverse Scale Space. As $t$ grows, our method progressively learns finer scale information, until fully recovering the original image. \textbf{(b)} Illustration of the difference between low-frequency components and large-scale features. The first row shows the image, and the second shows visualization via Grad-CAM \citep{selvaraju2017grad}. Unlike low-frequency images, large-scale images smooth out fine-grained details without blurring important features such as texture and object contours, effectively removing redundant information.}
    \label{fig.sol_path} 
\end{figure*}

To further enhance the robustness, we introduce several training procedures to integrate ViROLBI, including \emph{fixed training procedure} and \emph{iterative training procedure}. Specifically, the fixed training directly trains the model parameters on smoothed data with fixed sparsity; while the iterative training alternatively runs the instance smoothing algorithm and optimizes the model parameters. Critically, we can also apply the above procedure to tune any trained model. To validate the effectiveness of the proposed pipeline, we conduct extensive experiments on various types of noise, including adversarial noise, low-resolution images, and natural corruption. 

Our main contributions are summarized as follows.
\begin{itemize}[topsep=1pt]

    \item We introduce an image sparsification approach as a differential inclusion with Total Variation regularization, which can effectively remove small-scale features. 

    \item We introduce a sparse projection technique to derive the Total Variation regularized image, enhancing the algorithm's implementation efficiency.

    \item We present several training procedures that seamlessly integrate our sparsification algorithm into the training process, which further enhances the robustness of visual models.

    \item Our model demonstrates promising results in robustness tasks, including noisy classification, adversarial defense, and low-resolution classification. Additionally, we employ visualization to illustrate our method's capacity to capture semantic features.
\end{itemize}

\section{Related Work}

\textbf{Vision Robustness.} The robustness issue in vision tasks has been intensively studied. Most of these works only focused on specific types of noise, mainly including natural corruption \citep{momeny2021noise, mintun2021interaction}, adversarial noise \citep{szegedy2013intriguing, costa2024deep}, and low-resolution artifacts \citep{bulat2018learn, lugmayr2019unsupervised}. For each specific type considered, these works exploited properties of the noise and tailored their methods accordingly to achieve robustness \citep{laugros2022synthetic}. Specifically, since artifacts in natural corruption can be easily modeled, one could employ data augmentation \citep{hendrycks2019augmix, lopes2019improving, hendrycks2021many} or self-supervised learning \citep{sun2020test, wang2020tent} to improve robustness. The data augmentation method was applied to low-resolution images by generating resolution-degraded images with compressed artifacts \citep{bulat2018learn, kim2020unsupervised}. Similarly, for adversarial noise, it is common to perform adversarial training \citep{tramer2018ensemble, mummadi2019defending, kumari2019harnessing, wong2019fast} using data generated by adversarial attacks, with the difference from data augmentation that it is an end-to-end optimization.

There are also some works focused on general robustness, by either learning domain-invariant representation \citep{li2018deep, sun2021recovering} or identifying only low-frequency features for prediction \citep{yucel2023hybridaugment++, wang2023lfaa, yin2019fourier}. However, the representation learning methods commonly relied on un-pooled data from multiple domains, while smoothing high-frequency components will induce compression artifacts like low-resolution images, thus losing important features such as the contour and content of the object \citep{maiya2021frequency, wang2020high, zhang2022range}. In contrast, we provide a new perspective of inverse scale space, which can efficiently capture important and large-scale features while removing small-scale ones where corruptions frequently occur, thereby achieving robustness uniformly against various types of noise.

\textbf{Total Variation and Inverse Scale Space methods.} Total Variation (TV), proposed by~\citep{rudin1992nonlinear}, has been successfully applied in various vision tasks including denoising~\citep{11, 16},
deconvolution~\citep{17},
deblurring~\citep{11},
%inpainting~\citep{1},
superresolution~\citep{marquina2008image},
structure-texture decomposition~\citep{aujol2006structure},
and segmentation~\citep{25}. Recently, \citep{yeh2022total} has shown the benefit of deep learning models brought by the introduction of the TV Optimization layer. On the other hand, the \textit{inverse scale space methods} \citep{burger2005nonlinear, scherzer2001inverse} were firstly proposed in image denoising. Different from previous methods that smooth firstly small-scale features from a noisy image, these methods start from a blank image and progressively learn finer scale information, until successfully recover the clean image from the noisy one. Later, this property was integrated into an ordinary differential equation for sparse recovery \citep{ORXYY16, huang2016split}, as continuous limits of Linearized Bregman Iteration (LBI) \citep{yin2008bregman} for image processing. Besides, they established the \emph{model selection consistency} in this ODE, \emph{i.e.}, the important features are selected first than others.

In this paper, we explore the total variation sparsity in the inverse scale space, termed as inverse scale variational sparsification, such that the TV-regularized images, as solutions of a newly proposed differential equation at early iterations, are able to contain only large-scale information. 

\section{Methodology}
\label{sec.method}

\textbf{Problem Setup.} Given a clean dataset $\{x_i,y_i\}_{i=1}^n$, our goal is to learn a predictor $f: \mathcal{X} \to \mathcal{Y}$, where $\mathcal{X}$ represents the image space and $\mathcal{Y}$ represents the label space, such that it can generalize well to new data that may be corrupted by various types of noise.

To achieve this goal, we present a unified framework to effectively extract large-scale features for training. In Sec.~\ref{sec.method-iss}, we first introduce a differential equation induced by Total Variation regularization, as an image sparsification method to obtain $x_i^L$ with only large-scale information, for each training data $x_i$. Then in Sec.~\ref{sec.method-train}, we introduce our training procedures for learning robust neural networks.

\subsection{ Inverse Scale Variational Sparsification} 
\label{sec.method-iss}

To smooth out small-scale features from an image $x \in \mathbb{R}^p$ ($p:=h\times w$ denotes the size of the image vector, with $h,w$ \emph{resp.} denoting the height and width), we consider the following variation problem known as Rudin-Osher-Fatemi (ROF) functional \citep{rudin1992nonlinear, osher2005iterative} that was proposed in image denoising, 
\begin{equation}
\label{eq.TV}
    \min_u \frac{1}{2} \Vert u - x \Vert_2^2 + \lambda \Vert u \Vert_{BV}, 
\end{equation}
where $\Vert \cdot \Vert_{BV} := \Vert \nabla u \Vert_1$ denotes the bounded variation norm. In the image space $\Omega \subset \mathbb{R}^2$, this norm can be expressed as $\Vert Du \Vert_1$, where $D$ is a graph difference matrix associated with a graph $G:=(V,E)$ with $V:=\{1,...,p\}$ and $E$ denoting the set of adjacent pixels pairs, such that $(Du)(i,j):=u_i - u_j$ for each $(i,j) \in E$. Here, $\lambda$ denotes the scale parameter. A large-scale parameter will smooth out those small-scale features, which refer to non-significant variational differences among $u$. By varying $\lambda$ from $\infty$ to $0$, Eq.~\eqref{eq.TV} generates a smoothed image path $\{u(\lambda)\}$, where $u(\lambda)$ with larger $\lambda$ is more variational sparse, and consequently, more small-scale features are being smoothed out. However, the optimization to obtain such an image path is very expensive, especially when $p$ is large \citep{chan1999nonlinear}. 

To improve the efficiency, we introduce the variable splitting scheme \citep{goldstein2009split} into the following objective:
\begin{equation}
\label{eq.obj-split}
    \mathcal{L}_\beta(u,\gamma):= \frac{1}{2} \Vert u - x \Vert_2^2 + \beta \Vert Du - \gamma \Vert_2^2, \ \beta > 0.
\end{equation}
Here, $\gamma$ is an augmented parameter constrained to exhibit sparsity and proximity to $Du$, where the latter constraint is achieved through the $\ell_2$ term $\frac{\rho}{2} \Vert Du - \gamma \Vert_2^2$. To enforce sparsity on $\gamma$, we consider the following dynamics: 
\begin{subequations}
\label{eq:bregman-iss} 
\begin{align}
	0 & =-\nabla_{u}{\mathcal{L}_\beta}\left(u_{t},\gamma_{t}\right),\label{eq:bregman-iss-show-a}\\
	\dot{\rho_{t}} & =-\nabla_{\gamma}{\mathcal{L}_\beta}\left(u_{t},\gamma_{t}\right),\label{eq:bregman-iss-show-b}\\
	\rho_{t} & \in \partial\Vert\gamma_{t}\Vert_1. \label{eq:bregman-iss-show-c}
\end{align}
\end{subequations} 
\begin{remark}
    When $\frac{1}{2} \Vert u - x \Vert_2^2$ becomes the squared loss $\frac{1}{2} \Vert y - Xu \Vert_2^2$ for the linear model,  ($y$, $X$ denotes the response vector and covariate matrix), our dynamics in Eq.~\eqref{eq:bregman-iss} degenerates to the Split Bregman Inverse Scale Space \citep{huang18_acha} that was proposed for signal recovery from noisy measurements. In contrast, our goal is to smooth scale scale information from a clean image.  
\end{remark}
Starting from $u(0) = 0$ and $\gamma(0) = 0$, such a differential equation generates a regularization solution path as $t$ increases, where $\gamma(t)$ changes from sparse to dense, leading to a smoothed image flow $\widetilde{u}_t$ starting from $\widetilde{u}_0 = \bf{0}_p$ to $\lim_{t \to \infty} \widetilde{u}_{t} = x$. Therefore, $t$ plays a similar role to $1/\lambda$, and hence is called \emph{inverse scale space parameter}. That means, the image at a small $t$ will preserve only large-scale features while smoothing those small-scale ones. 
\begin{remark}
    To understand in details why $t$ is called the \emph{inverse scale space parameter}, we consider the following objective with the scale space parameter $\lambda$:
    \begin{equation}
    \label{eq.lasso-tv}
        \mathcal{L}_{\beta,\lambda}(u,\gamma) := \frac{1}{2} \Vert u - x \Vert_2^2 + \beta \Vert Du - \gamma \Vert_2^2 + \lambda \Vert \gamma \Vert_1.
    \end{equation}
    As the scale parameter $\lambda$ decreases from $\infty$ to 0, $\gamma_{\lambda}$ in Eq.~\eqref{eq.lasso-tv} gets from sparse to dense, with $\lim_{\lambda \to \infty} \gamma_\lambda = 0$ and $\lim_{\lambda \to 0} (u_\lambda,\gamma_\lambda) = \arg\min_{u,\gamma} \mathcal{L}_\beta(u,\gamma)$ in Eq.~\eqref{eq.obj-split}. 
    During this process, the solution path $\gamma_{\lambda}$ will progressively learn finer scale features as $\lambda$ decreases, which is similar to the behavior of $\gamma_t$ in Eq.~\eqref{eq:bregman-iss} as $t$ increases. Therefore, $t$ plays a similar role to $1/\lambda$, hence is called the inverse scale space parameter. 
\end{remark}

To explain, we note from Eq.~\eqref{eq:bregman-iss-show-b} that $\rho_t$ follows a gradient descent flow, starting from $\rho_0 = 0$. This implies that $\gamma_0 = 0$, according to the definition of the subgradient in Eq.~\eqref{eq:bregman-iss-show-c}. As $t$ grows, more elements $\rho_t \in \partial \Vert \gamma_t \Vert_1$ tend to hit the boundary of $\pm 1$, making corresponding elements of $\gamma_t$ being non-zeros according to Eq.~\eqref{eq:bregman-iss-show-c}. Because $\gamma_t$ is sparse at each $t$, we can obtain a variational sparse image $\widetilde{u}_t$ by projecting $u_t$ onto the subspace expanded by the support set of $\gamma_t$, \emph{i.e.}, $\widetilde{u}_t := \mathrm{Proj}_{S_t}(u_t)$ with $S_t:=\mathrm{supp}(\gamma_t):=\{i:\gamma_t(i)\neq 0\}$. After projection, $D_{S_t^c}\widetilde{u}_t = 0$, indicating that that $\widetilde{u}_t$ smooth out those non-significant variational differences outside $S_t$. Consequently, as $t$ increases, $S_t$ expands, enabling $\widetilde{u}_t$ to learn finer-scale features. Therefore, we need a proper $t_0$ to stop the dynamics, ensuring that small-scale features are removed in $\widetilde{u}_{t_0}$.

\begin{remark}
    $t$ is a trade-off between robustness and accuracy in the standard setting. With access to noisy data, we can determine $t_0$ through cross-validation based on the reconstruction error. In a general setting where noisy data are not accessible, we determine it according to the sparsity level of the smoothed image. Empirically, we observe that our algorithm consistently achieves robustness across a diverse range of sparsity levels, provided that it remains below 0.9.
\end{remark}

\textbf{Discretization for Implementation.} We follow \citep{yin2008bregman, huang18_acha} to propose an iterative form of Eq.~\eqref{eq:bregman-iss}, by introducing an elastic net penalty $\Vert \gamma_t \Vert_1 + \frac{1}{2\kappa} \Vert \gamma_t \Vert_2^2$ to approximate the original $\ell_1$ penalty $\Vert \gamma_t \Vert_1$ by setting $\kappa$ large enough. Let $z_t \in \partial_\gamma \left( \Vert \gamma_t \Vert_1 + \frac{1}{2\kappa} \Vert \gamma_t \Vert_2^2 \right)$, we have the following iteration:
\begin{subequations}
\label{eq:lbi}
\begin{align}
\quad & u_{k+1} = u_{k} - \kappa\alpha \nabla_u \mathcal{L}_\beta(u_{k},\gamma_{k}), \label{eq:lbi-beta} \\
& z_{k+1} = z_{k} - \alpha \nabla_\gamma \mathcal{L}_\beta(u_{k},\gamma_{k}), \label{eq:lbi-z} \\
% & \gamma_{k+1} = \kappa* \eta(z_{k+1},1), \label{eq:lbi-gamma}
& \gamma_{k+1} = \kappa * \mathrm{prox}_{\Vert \gamma \Vert_1}(z_{k+1}), \label{eq:lbi-gamma}
\end{align}
\end{subequations}
where $\mathrm{prox}_{\Vert \gamma \Vert_1}(z_t) := \arg\min_{u} \frac{1}{2}\|u-z_t\|^2+\Vert u \Vert_1 = \mathrm{sign}(z_t) \max(|z_t|-1,0)$ gives an explicit form of $\gamma_k$ from $z_k$. Here, $\alpha$ is the step size to approximate the gradient. If $\alpha \to 0$ and $\kappa \to \infty$, the above iteration will converge to the original dynamics Eq.~\eqref{eq:bregman-iss}. Besides, $\alpha$ should satisfy $\alpha < \frac{2}{\kappa \Vert H_\nu \Vert_2}$ with $H_\nu:=\nabla^2 \mathcal{L}_\beta(u,\gamma)$, in order to ensure that $\mathcal{L}_\beta(u_k,\gamma_k)$ decrease as iterates. We term this iteration as \textbf{Vi}sion \textbf{Ro}bust \textbf{L}inearized \textbf{B}regman \textbf{I}teration (ViRoLBI). Compared to running Eq.~\eqref{eq.TV} for several $\lambda$, this iteration can easily obtain a whole smoothed image path with a single run of Eq.~\eqref{eq:lbi}.

To obtain a smoothed image $\widetilde{u}_k$ at each $k$, we project $u_k$ onto the subspace of the support set of $\gamma_k$, \emph{i.e.}, $S_k := \mathrm{supp}(\gamma_k) := \{i: \gamma_k(i) \neq 0\}$, such that $D_{S_k^c}\widetilde{u}_k=0$:
\begin{equation}
\label{eq.projection}
    \widetilde{u}_{k}=\mathrm{proj}_{S_k}(u_k) := \arg\min_{D_{S_k^c}u^\prime=0} \Vert u^\prime -u_k \Vert_2.
\end{equation}
We can obtain the closed-form solution of Eq.~\eqref{eq.projection} as: $\widetilde{u}_k = (I - D_{S_k^c}^\dagger D_{S_k^c})u_k$\footnote{For a general matrix $A$, we denote $A_S$ as the sub-matrix of $A$ with rows indexed by $S$}, where $D_{S_k^c}^\dagger$ denotes the pseudo-inverse matrix of $D_{S_k^c}$. However, the computation is expensive, as the complexity of $D_{S_k^c}^\dagger$ is at the scale of $\mathcal{O}\left(|S_k^c|^3\right)$, which can be significantly higher than the gradient descent step's cost of $\mathcal{O}(p)$, especially when $\gamma_k$ is sparse in the early iterations. To accelerate, we propose an efficient projection algorithm by exploiting the graph structure of $D_{S_k^c}$. Specifically, note that $D_{S_k^c}$ corresponds to the sub-graph $G_{S_k^c}:=(V,E_{S_k^c})$, such that 
\begin{align*}
    D_{S_k^c}(\widetilde{u})(i,j) := \widetilde{u}_k(i) - \widetilde{u}_k(j) = 0, \ \forall (i,j) \in E_{S_k^c}. 
\end{align*}
That means, we have $\widetilde{u}_k(i) = \widetilde{u}_k(j)$ as long as $i$ and $j$ are connected. We then propose to identify all connected components, since $\widetilde{u}$ shares the same value for all elements within each component. To minimize Eq.~\eqref{eq.projection}, this shared value should be the average of $u_k$ over elements in that component. Since the complexity of finding connected components of a $p$-node graph is $\mathcal{O}(p)$, this projection has the same complexity as the gradient descent, as summarized below.
\begin{prop}
    \label{prop.graph}
    Given $u_k$ and $S_k:=\mathrm{supp}(\gamma_k)$, if $G=(V,E_{S_k^c})$ has $C$ connected components $G_1=(V_1,E_1),...,G_C=(V_C,E_C)$, such that $V=V_1 \cup ... \cup V_C$, then $\widetilde{u}_k$ in Eq.~\eqref{eq.projection} can be determined as follows, with a complexity of $\mathcal{O}(p)$:
    \begin{equation*}
        \widetilde{u}_k(j) = \overline{u}_k(V_c):=\frac{1}{|V_c|} \sum_{l \in V_c} u_k(l), \ \forall j \in V_c\ \text{ for some $c\in \{1,..,C\}$}.
    \end{equation*}
\end{prop}
% With such an efficient projection, we can efficiently obtain the solution path of $(u_k,\gamma_k,\widetilde{u}_k)$. 

\textbf{Extension to colored image via group sparsity.} For a colored image, we have $x \in \mathbb{R}^{p \times 3}$. This means each pixel is a 3-d vector $x_i = [x_{i1}, x_{i2}, x_{i3}]$ in the RGB channels. Correspondingly, we enforce group sparsity on $\gamma \in \mathbb{R}^{p \times 3}$, where each group $i$ corresponds to a vector $\gamma(i,) \in \mathbb{R}^3$: 
\begin{equation}
\negthickspace \negthickspace \negthickspace P(\gamma)=\Vert \gamma \Vert_{1,2}:=\sum_i \Vert\gamma(i,) \Vert_2 = \sum_i \sqrt{\gamma^2(i,1) + \gamma^2(i,2) + \gamma^2(i,3)}. 
\end{equation}
Let $z \in \partial_\gamma \left( P(\lambda) + \frac{1}{2\kappa} \Vert \gamma \Vert_2^2 \right)$, we obtain $\gamma_k$ from $z_k \in \mathbb{R}^{p \times 3}$ as follows: 

\begin{align}
    \gamma(i,)= \mathrm{prox}_{\Vert \gamma \Vert_{1,2}}(z)_i:=
\begin{cases}
    \left(1-\frac{1}{\Vert z(i,) \Vert_2} \right)z(i,) & \Vert z(i,) \Vert_2\ge 1, \\
0 & \text{otherwise,} 
\end{cases}
\end{align}
which can replace Eq.~\eqref{eq:lbi-gamma} to generate the smoothed image path for colored images.

\subsection{Robust Network Learning}
\label{sec.method-train}

We introduce two training strategies to learn robust neural networks: \emph{fixed training}, and \emph{iterative training}. Let $f_\theta : \mathcal{X} \to \mathcal{Y}$ as the neural network parameterized with $\theta$, where $\theta$ is typically trained via \emph{Empirical Risk Minimization} (ERM) with loss $\ell(f_\theta(x),y)$.

\textbf{Fixed training.} We directly train $f_\theta$ on smoothed images when the sparsity level of $\gamma$ (\emph{i.e.}, the proportion of non-zero elements of $\gamma$) reaches a fixed value, \emph{e.g.}, $80\%$. Since only large-scale information is preserved in data, we expect the trained network to smooth small-scale features. 

\textbf{Iterative training.} We train the network parameter $\theta$ and run ViRoLBI in an alternative manner, as shown below:
\begin{subequations} %subequations
\label{eq:iterative}
\begin{align}
u_{k+1} & = u_{k}- \kappa\alpha \nabla_u \mathcal{L}_\beta(u_{k},\gamma_{k}), \nonumber \\
z_{k+1}&=z_{k}-\alpha \nabla_\gamma \mathcal{L}_\beta(u_{k},\gamma_{k}), \nonumber \\
\gamma_{k+1}&=\kappa* \mathrm{prox}_{\Vert \gamma \Vert_1}(z_{k+1}) \text{ from Eq.~\eqref{eq:lbi}}, \label{eq:iterative-lbi}\\
\widetilde{u}_{k+1} & = \mathrm{proj}_{\mathrm{supp}(\gamma_{k+1})}(u_{k+1}) \text{ from Prop.~\ref{prop.graph}}, \label{eq:iterative-proj} \\
\theta_{k+1}&=\theta_k-\nabla_\theta \ell(f_\theta(\widetilde{u}_{k+1},y)), \label{eq:iterative-grad}
\end{align}
\end{subequations}
where Eq.~\eqref{eq:iterative-grad} can be implemented by other optimizers such as SGD or Adam. As iterates, the network will first learn large-scale features, followed by small-scale ones. Instead of training networks for each scale-level data like fixed training, we can efficiently obtain a family of neural networks that progressively learn finer-scale features. If we stop at a proper iteration, the network will also learn only large-scale features. Empirically, both training procedures perform well on various types of noise, including natural corruption, adversarial noise, and low-resolution images. 

During testing, we can implement ViRoLBI to preprocess test data at a chosen sparsity level, in order to align the distribution to the training data. 

\vspace{-0.2cm}
\section{Experiments}
\label{sec.experiments}

\textbf{Applications}.
By effectively separating large-scale structural information from intricate details, our approach shows promise in enhancing robustness and explainability. Empirically, we show the scenarios with natural corruptions, adversarial attacks, and low-resolution images. In addition, we also show the potential of our framework in high-frequency perturbations. The visualization result is also provided.

\textbf{Datasets.} In order to demonstrate the generalization and scalability of our method, extensive benchmarks are adopted including CIFAR10 \citep{krizhevsky2009learning}, CIFAR100 and ImageNet100 \citep{russakovsky2015imagenet, zang2022dlme, pinasthika2024sparseswin}.
Their noisy variants are built using the same methods as in \citep{hendrycks2019benchmarking}, which contains different kinds of noisy and corrupted images, for noisy robustness testing.

We offer visualization of the regularized image path of instances from of ImageNet \citep{deng2009imagenet} and COCO Dataset \citep{lin2014microsoft} in Appx.~\ref{sec.more_vis} due to the space limit.

\textbf{Backbone.} We use ResNet18 for CIFAR10 and CIFAR100 and ResNet50 \citep{he2016deep} for ImageNet100 in our experiments. To further demonstrate the impact on the transformer model, we leverage the ViT-tiny \citep{dosovitskiy2020image} model for all datasets.

\textbf{Competitors.} (1) \textbf{Vanilla}. We train models on original clean images from datasets.  (2) \textbf{Blur}. We train models on images preprocessed by Gaussian-Blur. For images from CIFAR10 and CIFAR100, we use kernel size as 3, strength as 2. For images from ImageNet100, we use kernel size as 7, strength as 2. (3) \textbf{TV-layer}. We train models following \citep{yeh2022total}. This method is only applied to ResNet. (4) \textbf{Fix}. We train models on a fixed set of smoothed images generated by our method. (5) \textbf{Iterative}. We train models following strategy in Eq.~\ref{eq:iterative}. Specifically, for ViT-tiny and ResNet50, we finetune them with pretrain weight.

\textbf{Selection of Early Stopping Time.}  An essential component of our method is selecting the early stopping time based on the sparsity level, which increases over time. This selection is highly dependent on the image size. Empirically, for Fixed and Iterative training, we stop at a sparsity level of 0.6 for images from CIFAR10 and CIFAR100, and at a level of 0.3 for images from ImageNet100, as larger image sizes are more susceptible to natural corruptions.  We further discuss it in Appx. \ref{sec.more-with-sparsity}.

\subsection{Robustness Against Natural Corruptions}
\label{sec.noisy}

\begin{table}
\renewcommand\arraystretch{1.2}
\caption{Classification results on noisy examples. For each model, we display accuracy on test data without preprocessing (first row) and with preprocessing (second row).}
\label{tab:noisy_class}
\normalsize
\centering
% \scriptsize
\setlength{\tabcolsep}{3pt}
\resizebox{0.9\linewidth}{!}{
% \begin{tabular}{@{}c|c|ccc|ccc|cccccccccccc|c}
\begin{tabular}{@{}c|c|ccccc|cccc|ccc}
\toprule[2pt]
\multirow{2}{*}{Model} & \multirow{2}{*}{Noise} &  \multicolumn{5}{c}{CIFAR10} \vline & 
\multicolumn{4}{c}{CIFAR100} \vline & 
\multicolumn{3}{c}{ImageNet100} \\
\cline{3-14}
 &  & Vanilla & Blur & TV layer  & Fix & Iterative & Vanilla & Blur & Fix & Iterative & Vanilla & Fix & Iterative \\
 \hline
 \multirow{6}{*}{ResNet} & \multirow{2}{*}{Gaussian} & 45.90 & 11.57 & 49.97 & 23.91 & 33.20 & 7.35 & 1.12 & 1.56 & 3.05 & 40.31  & 35.72 & 39.96 \\
 &  & 72.57 & 61.91 & 76.15 & 75.34 & \textbf{78.46} & 18.83 & 10.30 & \textbf{22.28} & 17.01 & 45.33  & 45.39 & \textbf{48.15}\\
 \cline{2-14}
  & \multirow{2}{*}{Shot} & 59.08 & 13.48 & 62.14 & 33.41 & 42.39 & 6.65 & 1.11 & 1.56 & 2.59 & 36.51  & 32.10 & 37.36 \\
  &  & 76.34 & 70.14 & 78.39 & \textbf{83.93} & 83.69 & 15.28 & 9.31 & 25.11 & \textbf{28.72} & 37.82  & 33.67 & \textbf{39.26} \\
   \cline{2-14}
  & \multirow{2}{*}{Impulse} & 51.43 & 14.16 & 58.75 & 34.08 & 35.62 & 10.06 & 1.15 & 2.48 & 3.22 & 31.54 & 26.70 & 34.28 \\
  &  & 51.30 & 56.18 & 59.15 & 69.51 & \textbf{71.86} & 10.06 & 9.04 & 20.87 & \textbf{25.62} & 34.26  & 30.21 &  \textbf{39.53}\\
  \hline
   \multirow{6}{*}{ViT} & \multirow{2}{*}{Gaussian} & 58.75 & 28.32 & - & 48.25 & 50.41 & 6.58 & 2.07 & 5.43 & 4.40 & 35.42 & 40.87 & 42.29 \\
 &  & 67.58 & 68.59 & - & 80.19 & \textbf{81.07} & 13.23 & 15.09 & 23.68 & \textbf{23.74} & 48.17  & \textbf{52.94} & 48.39 \\
 \cline{2-14}
  & \multirow{2}{*}{Shot} & 68.04 & 37.77 & - & 60.60 & 58.66 & 5.73 & 2.14 & 5.02 & 3.95 & 31.77 & 38.84 & 38.76 \\
  &  & 71.29 & 73.90 & - & 82.11 & \textbf{82.66} & 11.06 & 13.20 & \textbf{21.66} & 21.38 & 36.64 & \textbf{44.38} & 40.35 \\
   \cline{2-14}
  & \multirow{2}{*}{Impulse} & 61.75 & 61.95 & - & 69.02 & 59.16 & 11.87 & 6.02 & 10.28 & 8.50 & 30.44  & 37.48 & 38.79 \\
  &  & 69.73 & 64.74 & - & \textbf{72.96} & 71.82 & 13.24 & 12.54 & 18.89 & \textbf{19.76} & 35.57 & \textbf{45.23} & 44.09\\
\bottomrule[2pt]
\end{tabular}}
\end{table}

We consider noisy images from CIFAR10-C, CIFAR100-C and ImageNet100-C \citep{hendrycks2019benchmarking}, specifically with Gaussian noise, shot noise and impulse noise. For each noise, we consider 5 severity for CIFAR10 and CIFAR100, and 2 severity for ImageNet100, reporting the average accuracy. To explain the efficacy of our proposed method when dealing with noisy images, we compare our model with the vanilla model, Blur, and TV Layer on CIFAR10-C, CIFAR100-C, and with the vanilla method on ImageNet100-C to demonstrate the potential of our method on large-scale datasets. In the test stage, we consider two scenarios for all the methods: with and without preprocessing (preprocess test images via our instance smoothing algorithm in Eq.~\ref{eq:lbi} with sparsity 0.6 for CIFAR10 and CIFAR100, 0.3 for ImageNet100). For the Blur model, we blur the test image as preprocessing.

We present the classification accuracy in Tab.~\ref{tab:noisy_class}. Our models outperform other baselines across almost all noise types and datasets, even without the benefit of preprocessing. Additionally, our model stands out by achieving further improvement over others during the testing phase, where it refines the small-scale information through preprocessing to enhance performance. Furthermore, when employed as a preprocessing technique, our sparsification framework significantly enhances the accuracy of nearly all models across various types of noisy data.

\subsection{Robustness against Adversarial Attack}
\label{sec.exp-adv}

In this section, we show the robustness of our method against adversarial attacks. The attacked data are generated via commonly-used FSGM \citep{goodfellow2014explaining} and PGD \citep{madry2017towards}, whose details can be referred to Sec.~\ref{sec.more_adv}. During the test stage, similar to Appx.~\ref{sec.noisy}, we smooth each data at a sparsity level of 0.6 for CIFAR10 and CIFAR100, and 0.3 for ImageNet100.

We report the accuracy at strengths {$\varepsilon={8}/{255}$ and $\varepsilon={16}/{255}$} of all the datasets in Tab.~\ref{tab:adv1}, where $\varepsilon$ stands for the attack strengths on normalized images. Apart from these results, we additionally report the result of PNI \citep{he2019parametric} with ResNet18 as the backbone and $\varepsilon={8}/{255}$. 
% Additionally, we report the comparison of our models with baseline focus on adversarial defence in Tab.~\ref{tab:adv2} with ResNet18 and $\varepsilon={8}/{255}$. 
We first note that for all methods, applying our variational sparsification framework to preprocess test data can bring significant robustness improvement, which suggests its utility in smoothing noise components. Besides, it is also interesting to see that all variants of our methods can outperform the Vanilla method by a large margin, which can further demonstrate the utility of our robust learning framework.

\begin{table}
\renewcommand\arraystretch{1.2}
\caption{Classification results on adversarial examples. For each model, we display accuracy on test data without preprocessing (first row) and with preprocessing (second row).}
\label{tab:adv1}
\normalsize
\centering
% \scriptsize
\setlength{\tabcolsep}{3pt}
\resizebox{0.8\linewidth}{!}{
\begin{tabular}{@{}c|c|ccccc|ccccc|ccc}
\toprule[2pt]
\multirow{2}{*}{Model} & \multirow{2}{*}{Strength} &  \multicolumn{5}{c}{CIFAR10} \vline & 
\multicolumn{5}{c}{CIFAR100} \vline & 
\multicolumn{3}{c}{ImageNet100} \\
\cline{3-15}
 &  & Vanilla & Blur & PNI & Fix & Iterative & Vanilla & Blur & PNI & Fix & Iterative & Vanilla  & Fix & Iterative \\
 \hline
 \multirow{4}{*}{ResNet} & \multirow{2}{*}{8/255} & 37.95 & 3.66 & 41.07 & 21.04 & 18.67 & 2.15 & 0.55 & 20.78 & 1.57 & 2.21 & 21.18 & 5.10 & 18.22 \\
 &  & 49.45 & 60.22 & 51.21 & \textbf{62.72} & 60.39 & 14.61 & 36.57 & 23.06 & \textbf{36.73} & 33.06 & 24.41 & \textbf{44.82} & 32.00 \\
 \cline{2-15}
  & \multirow{2}{*}{16/255} & 23.97 & 5.10 & 26.05 &  11.37 & 12.87 & 1.69 & 0.55 & 8.26 &  0.93 & 1.25 & 15.82 & 8.20 & 17.32 \\
  &  & 42.27 & 25.91 & 43.35 & \textbf{48.51} & 45.51 & 9.36 & 12.73 & 13.39 & \textbf{23.16} & 19.78 & 21.32 & \textbf{27.30} & 22.34 \\
  \hline
  \multirow{4}{*}{ViT} & \multirow{2}{*}{8/255} & 16.63 & 14.62 & - & 24.30 & 12.76 & 6.74 & 3.46 & - & 5.18 & 3.57 & 4.64 & 1.14 & 0.82 \\
  &  & 40.81 & 51.95 & - &  \textbf{64.33} & 58.34 & 
     12.29 & 33.18 & - &  \textbf{37.42} & 32.22 & 13.46 & \textbf{33.34} & 11.72\\
 \cline{2-15}
     & \multirow{2}{*}{16/255} & 16.48 & 12.58 & - & 23.20 & 16.96 & 6.18 & 2.26 & - & 4.76 & 3.09 & 3.48 & 0.48 & 0.38 \\
 &  & 31.10 & 37.44 & - & \textbf{51.63} & 45.27 & 9.75 & 20.40 & - & \textbf{26.65} & 21.56 & 8.06 & \textbf{20.54} & 6.10 \\
\bottomrule[2pt]
\end{tabular}}
\end{table}

\subsection{Robustness against Low Resolution}
\label{sec.exp-res}

To illustrate the robustness of our method against low-resolution data, we apply our method to the task of classifying low-resolution images. We first downsample the original images and then upsample to the original size via the nearest interpolation. The smaller intermediate size will result in a lower-resolution image. 

The results are presented for different models with various datasets in Tab.~\ref{tab:res_res} for test data with different scaling factors. As shown, all variants of our methods outperform the vanilla model with lower-resolution images, especially the Fixed training model. This result suggests the effectiveness of our sparsification framework in learning large-scale information during training, as the low-resolution images can smooth out the details while maintaining the object's shape and contour. 

\begin{table}
\renewcommand\arraystretch{1.2}
\caption{Classification results on low-resolution examples. "IN100" refers to the ImageNet100 dataset, and "scale factor" denotes the ratio of the compressed image size to the original size.}
\label{tab:res_res}
\normalsize
\centering
% \scriptsize
\setlength{\tabcolsep}{3pt}
\resizebox{0.8\linewidth}{!}{
% \begin{tabular}{@{}c|c|ccc|ccc|cccccccccccc|c}
\begin{tabular}{c|cc|cc|cc|cc|cc|cc}
\toprule[2pt]
Method & \multicolumn{6}{c}{ResNet} \vline & \multicolumn{6}{c}{ViT} \\
\hline
Dataset & \multicolumn{2}{c}{CIFAR10} & \multicolumn{2}{c}{CIFAR100} & \multicolumn{2}{c}{IN100} \vline & \multicolumn{2}{c}{CIFAR10} & \multicolumn{2}{c}{CIFAR100} & \multicolumn{2}{c}{IN100} \\
\hline
Scale Factor & 1/4 & 1/2 & 1/4 & 1/2 & 1/7 & 1/4 & 1/4 & 1/2 & 1/4 & 1/2 & 1/7 & 1/4 \\
\hline
Vanilla & 23.33 & 38.50 & 5.25 & 19.48 & 10.16 & 28.60 & 16.54 & 42.27 & 3.83 & 17.43 & 6.53 & 24.18 \\
\hline
Blur & 29.69 & 33.03 & 6.43 & 8.30 & 8.68 & 16.82 & 25.32 & 62.76 & 4.82 & 20.75  & 10.94 & 21.36 \\
% \hline
% Haug\citep{yucel2023hybridaugment++} & 21.81 & 80.62 & 6.47 & 49.12 & - & - & - & - & - & - & - & - \\
\hline
Fix & 30.00 & 41.12 & 7.97 & 14.53 & \textbf{13.24} & 29.53 & 24.52 & 64.64 & 6.90 & 25.58 & 12.28 & 25.76 \\
\hline
Iterative & \textbf{32.00} & \textbf{48.68} & \textbf{13.13} & \textbf{23.33} & 11.50 & \textbf{31.16} & \textbf{27.95} & \textbf{65.25} & \textbf{9.98} & \textbf{26.92} & \textbf{17.52} & \textbf{26.52} \\

\bottomrule[2pt]
\end{tabular}}
\end{table}

\subsection{Extension to High-frequency Perturbations}

To further demonstrate the capability of our method in defending general noise, we apply our method to high-frequency perturbed data, by following the scenario of \citep{wang2020high}. Specifically, we first decompose the images into low-frequency and high-frequency components as shown in Fig.~\ref{fig:hl_comp} (a), and then respectively test the accuracy of models on both the high and low-frequency components. 

\begin{figure}
\centering
\includegraphics[width=0.7\textwidth]{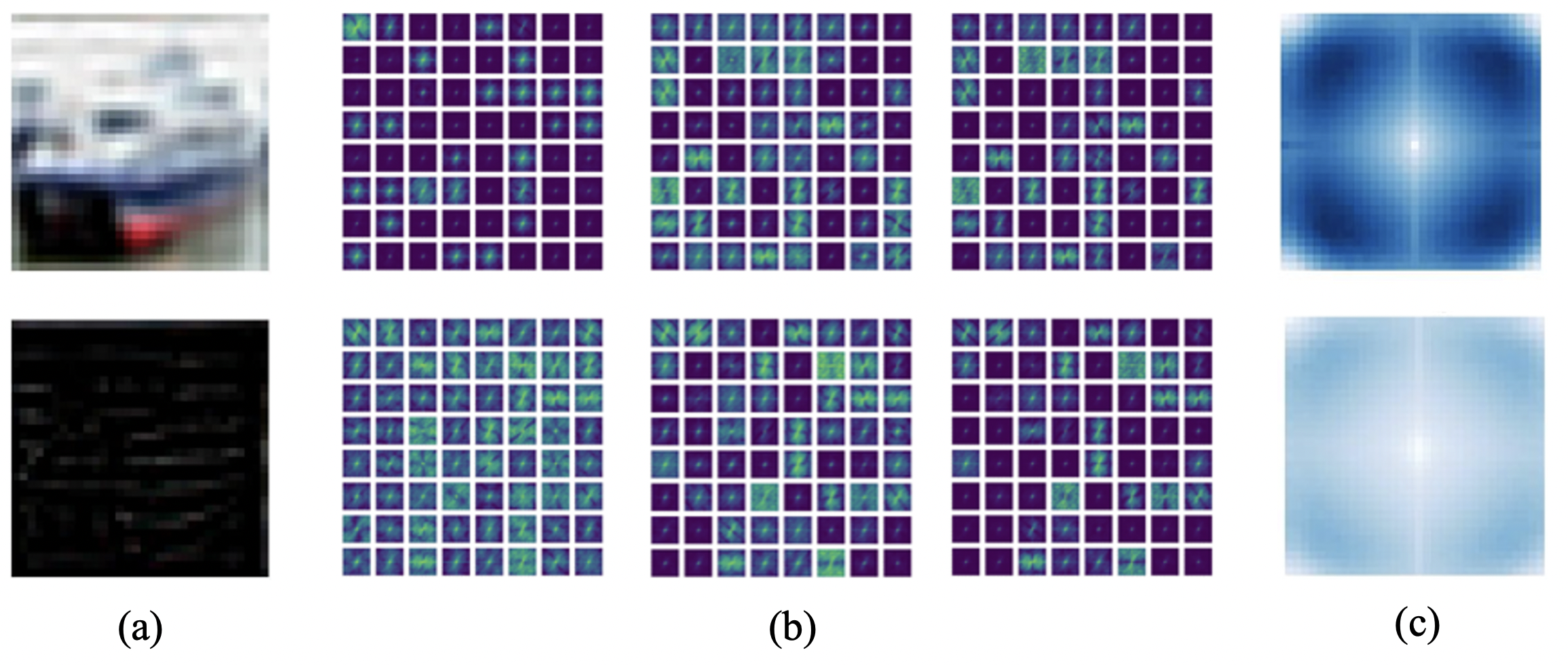}
\caption{\textbf{(a)} An example from CIFAR10 with a cut-off radius $r=6$ (above: low frequency component; below: high frequency component). \textbf{(b)} The feature map of the first convolution layer of ResNet (above: Iterative; below: Vanilla) in Epoch 9 (left), Epoch 59 (middle), and Epoch 159 (right). \textbf{(c)} The expected difference in the frequency domain on CIFAR10: the top (\emph{resp.} bottom) row shows the difference between the original image and the one with sparsity 0.6 (\emph{resp.} 0.8). \label{fig:hl_comp}  }

\vspace{-0.15in}
\end{figure}

As indicated in Tab.~\ref{tab:freq}, our models perform better than the vanilla model on low-frequency components, suggesting the capability of our framework to smooth high-frequency information. As a further verification, we visualize the frequencies in the first layer's feature maps during training in Fig.~\ref{fig:hl_comp} (b). As shown, the vanilla model (bottom) tends to learn high-frequency features while our method can first learn low-frequency features and then high-frequency features during training. Apart from that, in Fig.~\ref{fig:hl_comp} (c) we visualize the expected difference in the frequency domain as proposed in \citep{yin2019fourier}. To be concrete, we calculate $\mathbb{E}(\mathcal{F}(X) - \mathcal{F}(\hat{X}))$ on CIFAR10, where $\mathcal{F}$ stands for Fourier transformation, $X$ and $\hat{X}$ stand for different images. One can find that if we stop at a higher sparsity level, more high-frequency features are learned. With early stopping, the images contained more low-frequency features, which further supports the ability of our method to preserve low-frequent information while smoothing some high-frequency ones. 

\begin{table}
\renewcommand\arraystretch{1.2}
\caption{Test accuracy on high and low-frequency components of images.}
\label{tab:freq}
\normalsize
\centering
% \scriptsize
\setlength{\tabcolsep}{3pt}
\resizebox{0.7\linewidth}{!}{
\begin{tabular}{@{}c|c|cccc|cccc}
\toprule[2pt]
\multirow{2}{*}{Model} & \multirow{2}{*}{Strength} &  \multicolumn{4}{c}{CIFAR10} \vline & 
\multicolumn{4}{c}{CIFAR100}  \\
\cline{3-10}
 &  & vanilla & Blur & fix & iterative & vanilla & Blur & fix & iterative  \\
 \hline
 \multirow{2}{*}{ResNet} & High & \textbf{41.05} & 17.94 & 20.93 & 11.08 & \textbf{4.64} & 1.97 & 2.01 & 1.98 \\
 & Low & 47.37 & \textbf{93.14} & 73.29 & 73.63 & 13.69 & \textbf{72.91} & 41.96 & 48.30 \\

 \hline
  \multirow{2}{*}{ViT} & High & \textbf{35.90} & 29.23 & 33.09 & 19.67 & \textbf{8.75} & 3.61 & 4.46 & 4.07 \\
  & Low & 84.04 & \textbf{93.11} & 82.96 & 84.97 & 25.41 & \textbf{76.88} & 54.61 & 57.94 \\

\bottomrule[2pt]
\end{tabular}}
\vspace{-0.15in}
\end{table}

\subsection{Visualization}
\label{sec.model_exp}

We visualize learned features via Layer-Wise Relevance Propagation (LRP)~\citep{bach2015pixel}, which indicates the importance of features by backpropagating the relevance and overlaying the normalized results onto the input image. Appx.~\ref{sec.more-exp} presents more results implemented via Grad-CAM visualization \citep{selvaraju2017grad}.

The results regarding three image variants including original images, blurred images, and sparse images generated by our algorithm are shown in Fig.~\ref{fig:lrp}, where brighter colors typically indicate higher importance.

It is evident that on the original image, all the models can capture the main object (the bird), but the pixel in the background is important. After being preprocessed with our method, all the models can focus more on the main object without overly relying on the texture on the background.

\begin{figure}[htb]
\centering
\includegraphics[width=\textwidth]{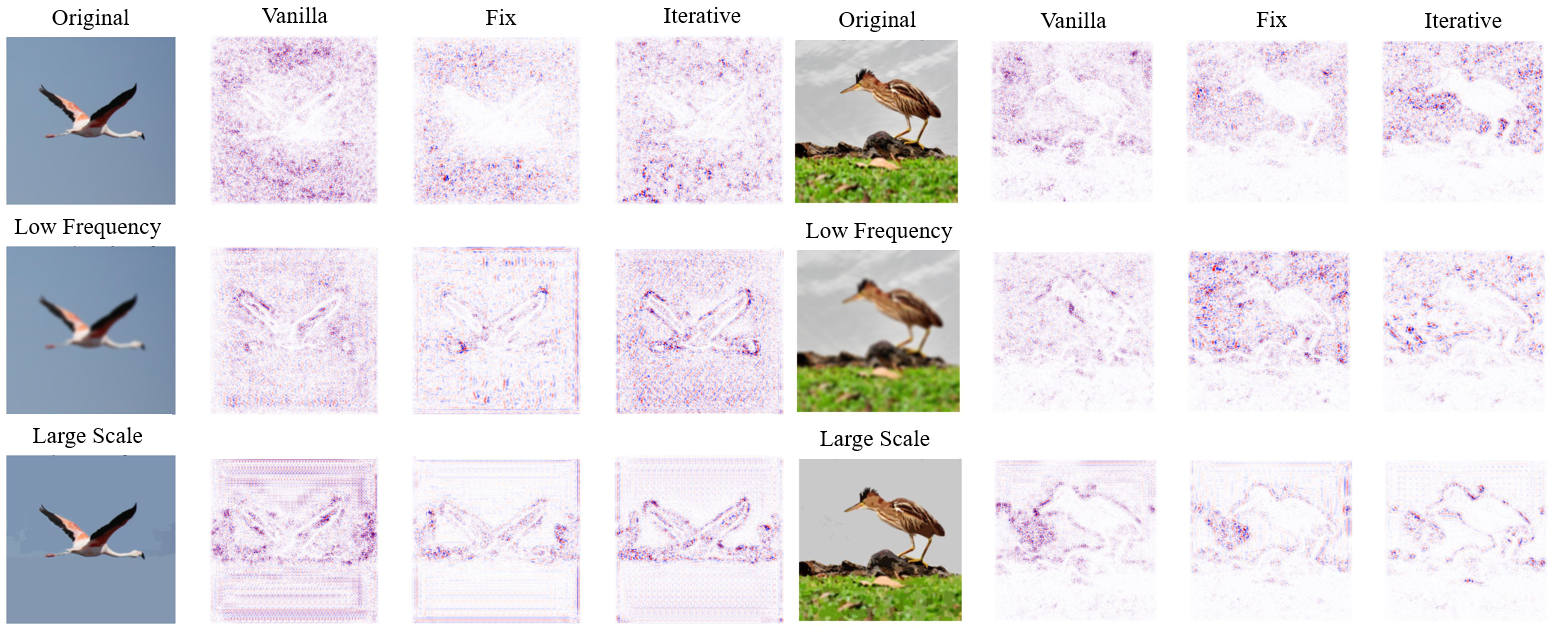} 
\caption{The LRP result of two images from ImageNet100. The first, second, and third rows show results for original images, Gaussian blurred images (containing only low-frequency components), and variational sparse images (containing only large-scale information), respectively. Brighter colors typically indicate higher importance. \label{fig:lrp} }

\vspace{-0.15in}
\end{figure}

\section{Conclusions and Discussions}
\label{sec.conclusion}

We propose a variational sparsification algorithm that exploits the Total Variational sparsity in the inverse scale space. By employing early stopping, this method efficiently smooths out small-scale features where noise typically occurs. Besides, it can effectively preserve important high-contrast features. With discretization and sparse projection, it has a simple iterative algorithm to implement. We demonstrate the utility in several robustness tasks.

\textbf{Limitations.} Although the complexity of sparse projection is comparable to gradient descent, its running time is significantly longer due to the current CPU-only implementation of the sparse projection algorithm. We believe our method can be applied to feature maps with TV regularization. Future work will explore this extension and optimize memory usage.

\clearpage
\bibliographystyle{unsrt}
\bibliography{reference}

\newpage
\appendix
\section{Proof of Proposition \ref{prop.graph}}
\label{sec.proof}

\textbf{Notations.} We project $\beta_k$ onto the sparse subspace of $\gamma_k$, \emph{i.e.}, $S_k := \mathrm{supp}(\gamma_k)$: $\tilde{\beta}_{k}=\mathrm{proj}_{S_k}(\beta_k) := \arg\min_{D_{S_k^c}\beta^\prime=0} \Vert \beta^\prime -\beta_k \Vert_2$. We denote $D_{S_k^c}$ as the sub-matrix of $D$ with rows indexed by $S_k^c$, which indexes the nonzero elements of $\gamma_k$. Specifically, $D_{S_k^c}$ is the graph difference matrix of $G:=(V,E_{S_k^c})$, such that $(i,j) \in E_{S_k^c}$ if $D_{S_k^c}\tilde{\beta}_k(i,j) := \tilde{\beta}_k(i) - \tilde{\beta}_k(j) = 0$. $E$ is the corresponding edge set of $D$, which is composed of adjacent pairs of pixels. 

\begin{proof}[Proof of Prop.~\ref{prop.graph}]
    Suppose $G=(V,E_{S_k^c})$ has $C$ connected components $G_1=(V_1,E_1),...,G_C=(V_C,E_C)$, such that $V=V_1 \cup ... \cup V_C$. If two nodes $i$ and $j$ are in the same component, the corresponding elements of $\tilde{\beta}_k$ have the same value, \emph{i.e.}, $\tilde{\beta}_k(i) = \tilde{\beta}_k(j)$. Then for each component $V_c$, $\tilde{\beta}_k(V_c)$ shares the same value. If we denote it as $\eta_c$, then the $\eta_c$ to minimize 
\begin{align*}
    \sum_{j \in V_c} (\eta_c - \beta_k(j))^2,
\end{align*}
equals to the average of $\beta_k(V_c)$, \emph{i.e.}, $\mathrm{mean}(\beta_k(V_c))$. Using the strong connected-component algorithm proposed in \citep{lulli2016fast}, the decomposition of connected components will cost $\mathcal{O}(log(p))$. 
\end{proof}

The algorithm is shown in Alg.~\ref{alg.graph}, and the flowchart of the graph algorithm is shown in Fig.~\ref{fig.graph_alg}.

\begin{algorithm}[ht]
\caption{Projection by Connected Components in Graph}
  \label{alg.graph}
\begin{algorithmic}% [1]
  \STATE {\bfseries Input:} An image $\beta$, current $\gamma_t$, the graph $G(V, E)$ where $V$ denotes the set of pixels and $E$ contains edges defined according to the graph difference matrix $D$ in Eq.~\eqref{eq.TV}.
  \STATE {\bfseries Output:} $\tilde{\beta}$ via projection in Eq.~\eqref{eq.projection}.
  \STATE Find connected components $G_1:=(V_1, E_1), \dots, G_C:=(V_C, E_C)$. 
  \STATE For each $i=1,...,C$, compute the average of $\beta$ over $V_i$, \emph{i.e.}, $z_i := \sum_{j \in V_i} \beta(j) / |V_i|$ and take $\tilde{\beta}(j)=z_i$ for each $j \in V_i$. 
  \STATE {\bfseries Return:} $\hat{\beta}$.
\end{algorithmic}
\end{algorithm}

\begin{figure}[ht!]
\centering
\includegraphics[width=\textwidth]{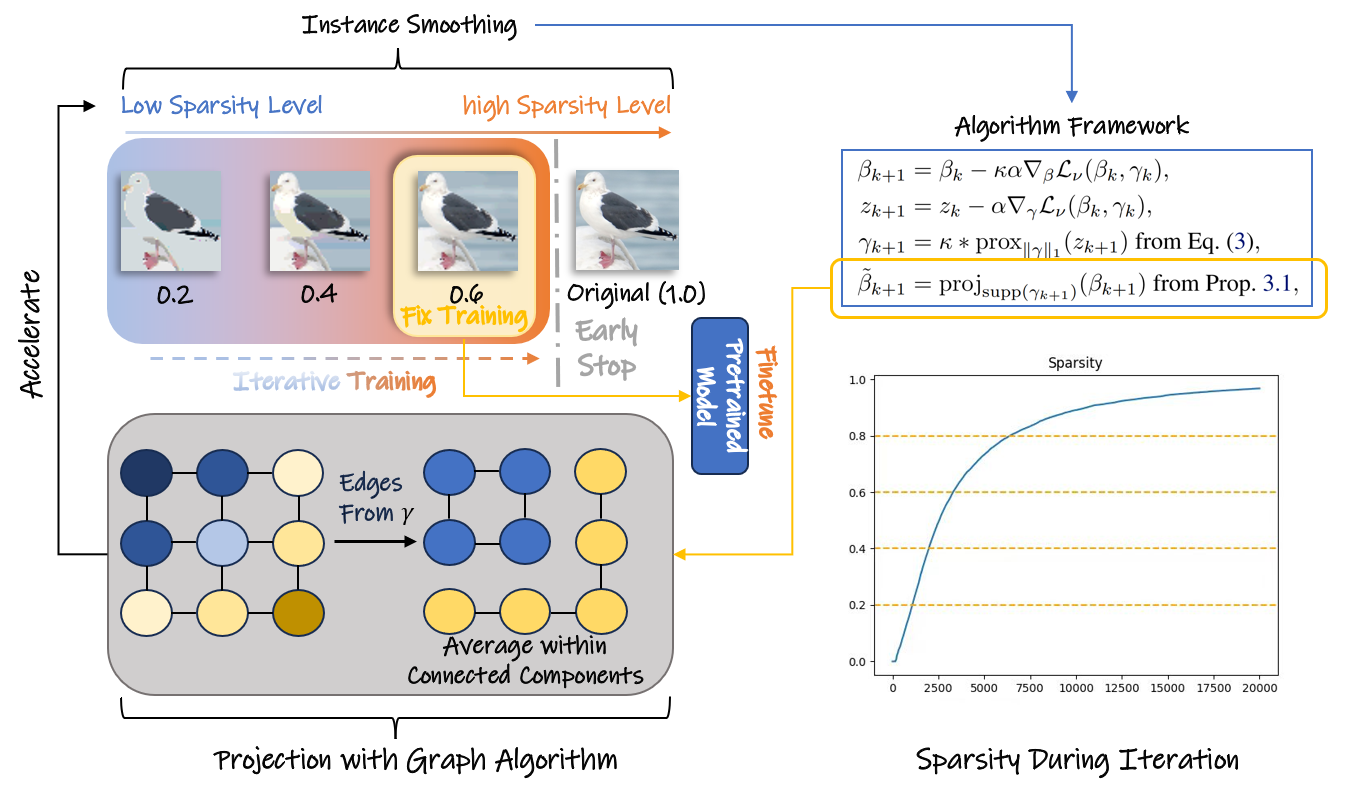}
\caption{Illustration of our training procedure and the Graph Algorithm used in acceleration.}
\label{fig.graph_alg} 
\end{figure}

\begin{remark}
    After obtaining the connected components, we need to compute the average of each component, which has the complexity of $\mathcal{O}(p)$ and is comparable to the gradient descent. Since the complexity of the soft-thresholding in Eq.~\eqref{eq:lbi-gamma} is also $\mathcal{O}(p)$, the overall complexity of our instance smoothing algorithm in Eq.~\ref{eq:lbi} has the same order of the gradient descent. 
\end{remark}

\section{Implementation Details}

We use ResNet18 for CIFAR10 and CIFAR100 and ResNet50 \citep{he2016deep} for ImageNet100. To further demonstrate the impact on the transformer model, we leverage the ViT-tiny \citep{dosovitskiy2020image} model for all the datasets. Specifically, we use the pre-trained ResNet50 and ViT-tiny \footnote{The weights are available in \url{https://huggingface.co/WinKawaks/vit-tiny-patch16-224}.} and further fine-tune them on our dataset, also with fixed, iterative, and incremental training strategies. For hyperparameters, we set $\kappa=5$, $\nu=1$, and $\alpha = \frac{1}{\kappa \Vert H \Vert_2}$, where $H=\nabla^2\mathcal{L}_\nu$ is the Hessian matrix of the loss function. We denote "Vanilla" for the vanilla model, "Fix" for the fixed training model, and "Iterative" for the iterative training model.

\section{Classification on clean data}
\label{sec.more-with-sparsity}

In this experiment, we apply our method to the standard classification on CIFAR10 and ImageNet100 \citep{deng2009imagenet}. 

\textbf{Experiment Setup. } We adopt ResNet-18 as the backbone on CIFAR10 and ResNet-50 on ImageNet100. For fixed training, we try different sparsity levels for preprocessing images. For the iterative training, we use the strategy in Eq.~\ref{eq:lbi} from the sparsity level 0.3 to 0.8.

\textbf{Results.} We report the classification accuracy on CIFAR10 In Tab.~\ref{tab:more-cls-10}. As shown, both fixed and iterative training offer comparable results to the vanilla model after the sparsity reaches $0.6$, suggesting that the loss in information is limited. Moreover, fixed training slightly outperforms the vanilla model at the sparsity level of $0.9$. We also report the classification accuracy on ImageNet100 In Tab.~\ref{tab:img100}.

\textbf{Further Discussion of Sparsity Level. } As shown in results in Tab.~\ref{tab:more-cls-10} and Tab.~\ref{tab:img100}, images with the sparsity level higher than 0.6 capture important features for classification. Also, from the results in Sec.~\ref{sec.noisy}, ~\ref{sec.exp-adv}, ~\ref{sec.exp-res}, we can empirically notice that training with those images can improve the robustness of the model. 

When the images are processed with a sparsity level of 0.6, the fine-grained small-scale information has been eliminated, while keeping the structural information. When the sparsity level is lower, only shape information (smooth information) is maintained, but some detailed semantic information. However, when the sparsity level is close to 1.0, noise and confusing texture will show up, which will deteriorate the robustness.

\begin{table}
\centering
\caption{Results of clean data in CIFAR10.}
\label{tab:more-cls-10}
\resizebox{\linewidth}{!}{
\begin{tabular}{c|cccccccccc|c}
\toprule[2pt]
Sparsity Level & 0.4 & 0.5 & 0.6 & 0.7 & 0.8 & 0.9 & 1.0 (Vanilla) & Iterative \\ \midrule
Accuracy & $85.24\pm 1.69$ & $92.47\pm 1.53$ & $93.82\pm 0.68$ & $94.53\pm 0.27$ & $94.78 \pm 0.24$ &$\textbf{95.38}\pm 0.14$& $95.29\pm 0.06$ & $94.54\pm 0.14$ \\
\bottomrule[2pt]
\end{tabular}}
\end{table}

\begin{table}
\centering
\caption{{Results of clean data in ImageNet100.}}
\label{tab:img100}
\resizebox{0.5\linewidth}{!}{
\small
\begin{tabular}{c|ccccccc|c}
\toprule[2pt]
Sparsity Level &0.4 & 0.6 & 0.8  & 1.0 (Vanilla) & Iterative \\ \midrule
Accuracy  & $59.12 $ & $75.51$ &${78.66}$& $79.36$& $74.39$ \\
\bottomrule[2pt]
\end{tabular}}
\end{table}

% \label{sec.exp-vis-feature}
\section{More Visualization Results}
\label{sec.more-exp}

\subsection{Visualizations with Grad-CAM}
In this experiment, we apply the Grad-CAM \citep{selvaraju2017grad} to visualize learned features during \emph{iterative training}. We consider the model trained using the strategy described in Eq.~\ref{eq:bregman-iss}, with sparsity levels ranging from 0.3 to 0.8. As shown in Fig.~\ref{fig:grad_cam}, the features learned by our model in the early epochs are more concentrated on the class-dependent regions (\emph{e.g.}, the cat's face in the top-left image and the dog`s body in the bottom left image). As iterates, finer-scale information is learned; thus the feature map is enlarged due to the completeness of information. The larger saliency map of our model shows that our model learns more shape information than the vanilla model.

\begin{figure}[ht!]
\centering
\includegraphics[width=\textwidth]{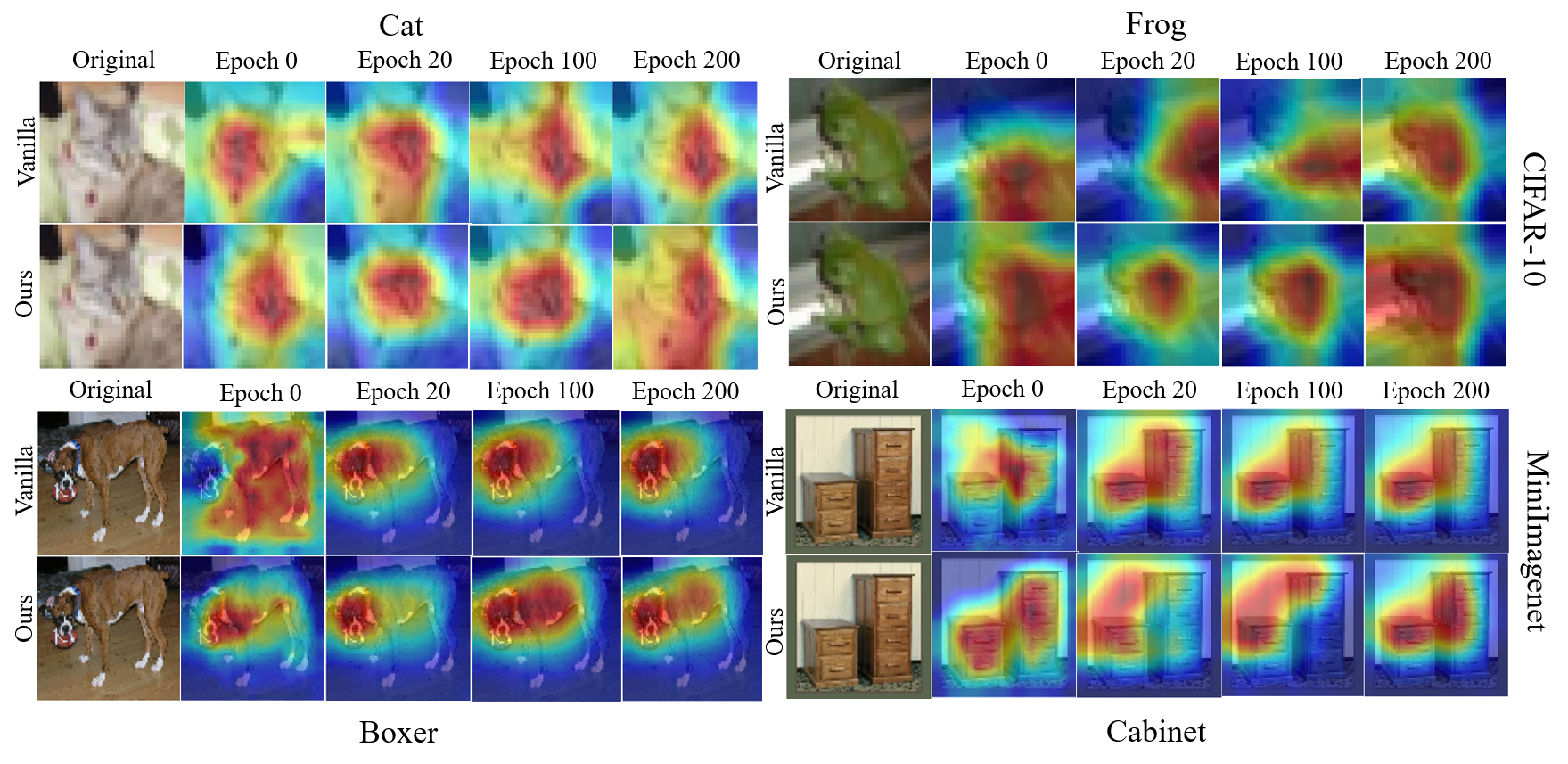}
\caption{Visualization of learned features in four images: cat (top-left), Boxer (bottom-left), Frog (top-right), and Cabinet (bottom-right) during \emph{iterative training}. The top two are from CIFAR-10 and the bottom two are from miniImagenet. In each image, the top and the bottom rows respectively correspond to the vanilla model and our method in Eq.~\eqref{eq:iterative}. \label{fig:grad_cam} }
\end{figure}

\section{Running Time of the Algorithm}
In this experiment, we compare the running time of our algorithm on gray-scale images from miniImagenet dataset to sparsity level 0.6. We consider the matrix factorization method and our graph method for the sparse projection in Eq.~\ref{eq:iterative-proj}. We run this test on an NVIDIA Tesla V100 (32GB) and an Intel Gold 6240 CPU @ 2.60GHz.

\textbf{Results.} For other methods such as Singular value decomposition (SVD) decomposition or QR decomposition that can obtain the closed-form solution suffer from high computational costs. Assume $p$ is the dimension of $\beta_k$. For example, the complexity of SVD decomposition is $\mathcal{O}(p^3)$, which is much more expensive than the gradient descent. In contrast, the complexity of the graph projection is only $\mathcal{O}(p)$.  To illustrate, we compare our graph projection methods with other alternatives, as well as the gradient descent in terms of time complexity. We report the running time for 15,000 iterations on a 84x84 grayscale image, in Tab~\ref{tab:time}. As shown, our graph projection method is much more efficient than others. 

\begin{table}
% \label{tab:running-time}
\tiny
\renewcommand\arraystretch{0.5}
\centering
\caption{Computational time ($s$) of different methods for 15,000 iterations on a 84x84 grayscale image.}
\label{tab:time}
\resizebox{0.6\linewidth}{!}{
\begin{tabular}{ccccc}
\toprule[1pt]
Projection Method & SVD & LSQR & Graph Algorithm \\
\cmidrule{1-4}
Running Time (s)   & $373.24\pm1.41$ & $171.79\pm2.31$ & $4.87\pm 0.09$  \\
\bottomrule[1pt]
\end{tabular}}
\end{table}

\section{Total Variation regularized Image Path}
\label{sec.more_vis}

In this experiment, we choose more cases in ImageNet  \citep{deng2009imagenet} and COCO Dataset \citep{lin2014microsoft}, a multi-object image dataset to visualize the regularized image path.

\textbf{Results.} 
As shown in Fig.~\ref{fig.path_shape}, as the sparsity level increases, the image first identifies large-scale structural information and then small-scale detailed information. Such large-scale information can refer to the object's shape or contour in the first three rows where the object as a whole has a convex and smoothed boundary; while in the last three rows with irregular and complex contour, such structural information can refer to the key parts of the object, \emph{e.g.}, the plow of a plow truck in the fifth row, and umbrellas in the last row. 

As shown in Fig.~\ref{fig.sol_path_coco}, when our method meets multi-object images, the shape of the object in the images will pop out at the beginning of the image path, and more detailed texture will gradually add to the background, and the object smoothly.

\begin{figure}[ht!]
  \centering
  \begin{tabular}{ccccc}
    \includegraphics[width=0.16\linewidth]{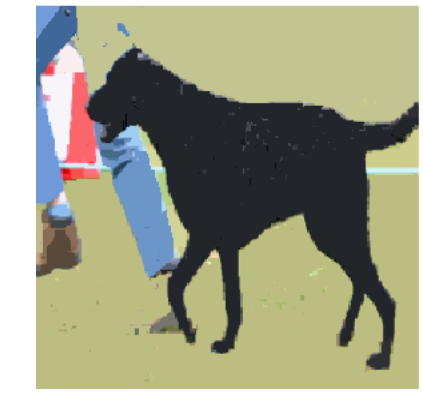} &
    \includegraphics[width=0.16\linewidth]{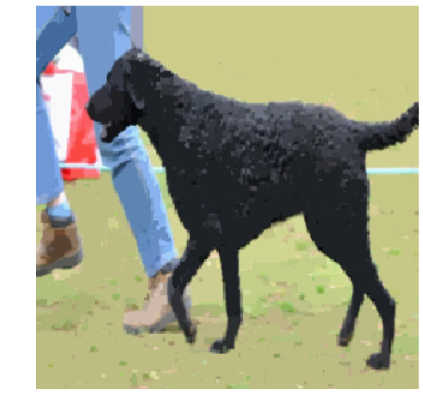} &
    \includegraphics[width=0.16\linewidth]{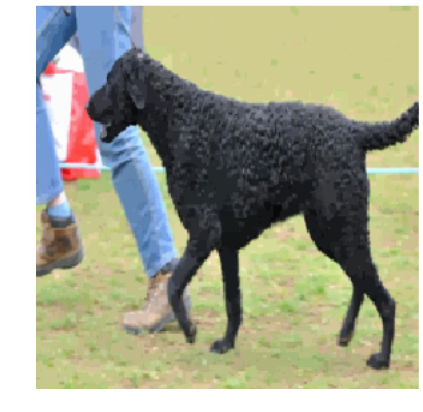} &
    \includegraphics[width=0.16\linewidth]{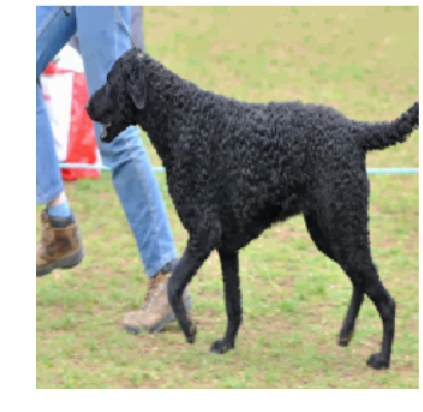} &
    \includegraphics[width=0.16\linewidth]{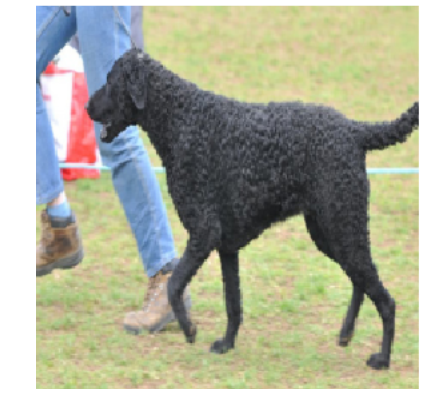} \\
    \includegraphics[width=0.16\linewidth]{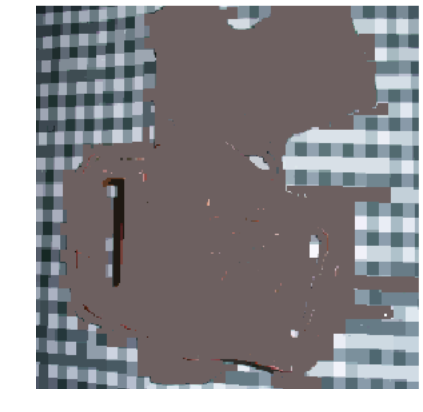} &
    \includegraphics[width=0.16\linewidth]{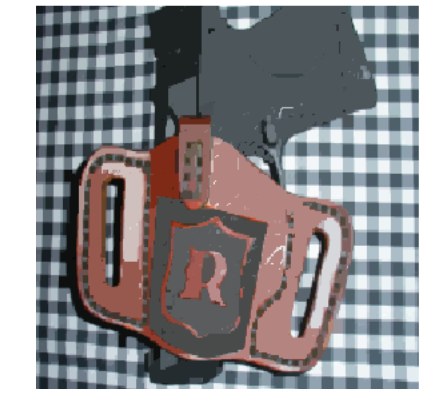} &
    \includegraphics[width=0.16\linewidth]{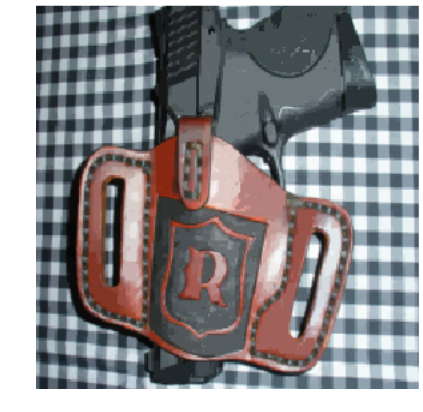} &
    \includegraphics[width=0.16\linewidth]{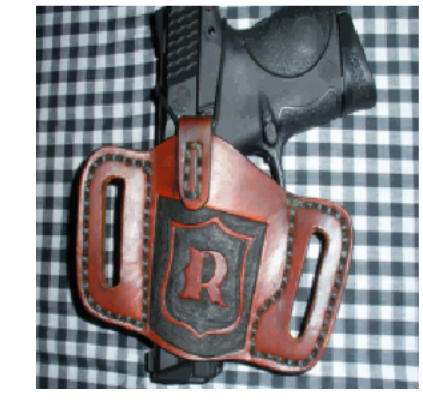} &
    \includegraphics[width=0.16\linewidth]{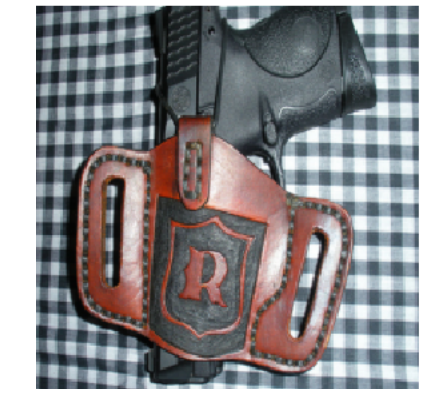} \\
    \includegraphics[width=0.16\linewidth]{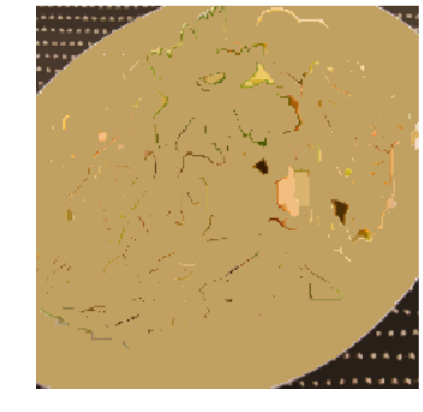} &
    \includegraphics[width=0.16\linewidth]{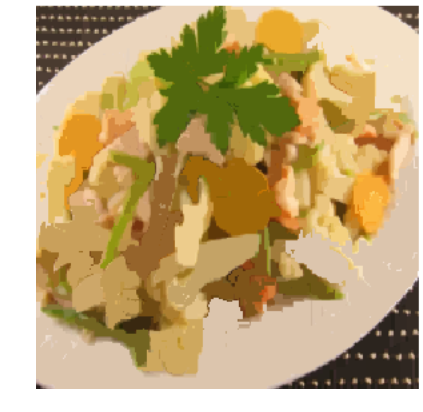} &
    \includegraphics[width=0.16\linewidth]{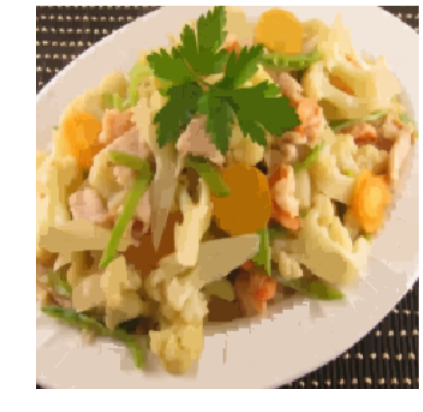} &
    \includegraphics[width=0.16\linewidth]{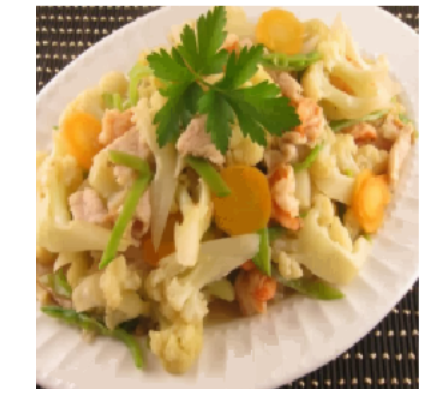} &
    \includegraphics[width=0.16\linewidth]{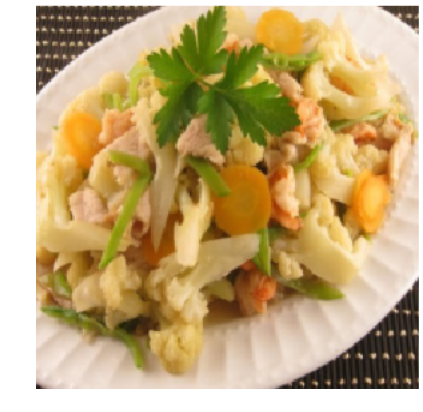} \\
    \includegraphics[width=0.16\linewidth]{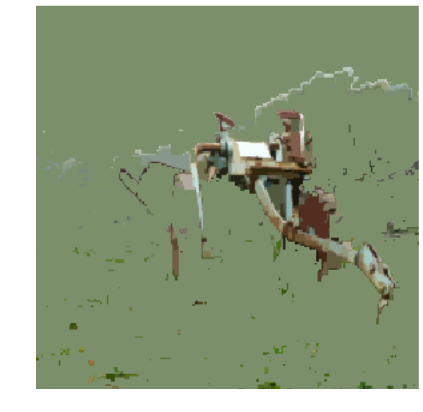} &
    \includegraphics[width=0.16\linewidth]{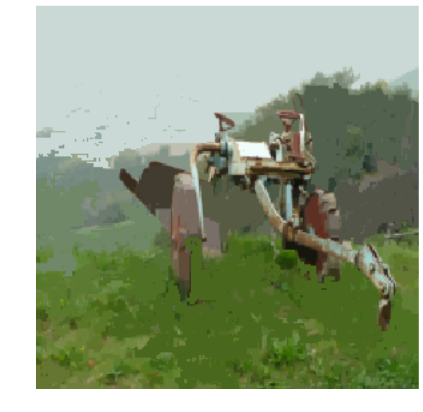} &
    \includegraphics[width=0.16\linewidth]{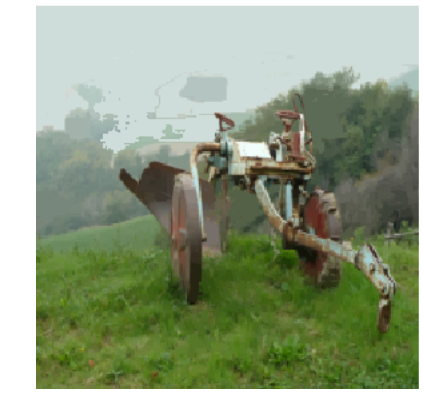} &
    \includegraphics[width=0.16\linewidth]{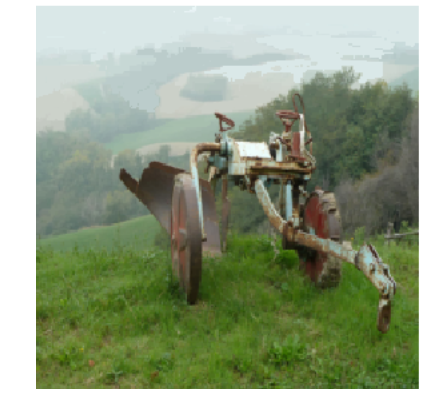} &
    \includegraphics[width=0.16\linewidth]{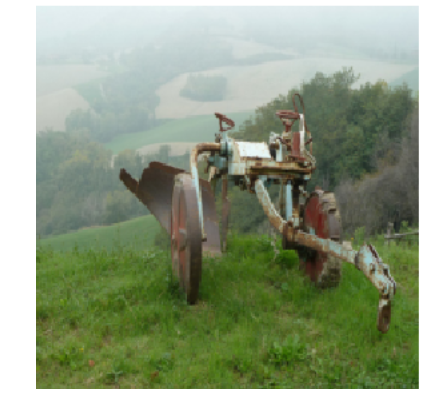} \\
    \includegraphics[width=0.16\linewidth]{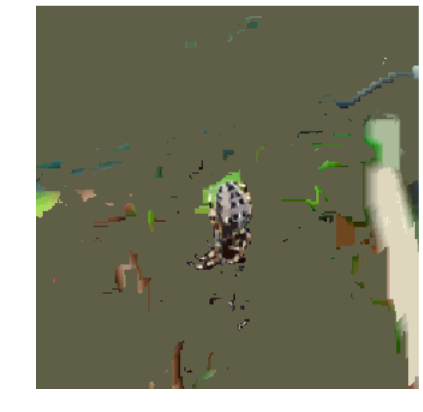} &
    \includegraphics[width=0.16\linewidth]{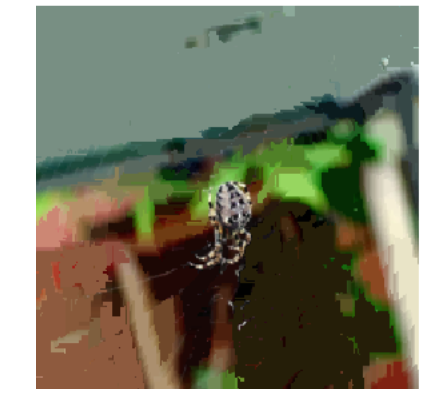} &
    \includegraphics[width=0.16\linewidth]{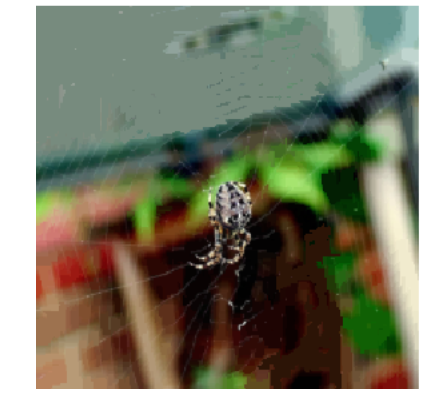} &
    \includegraphics[width=0.16\linewidth]{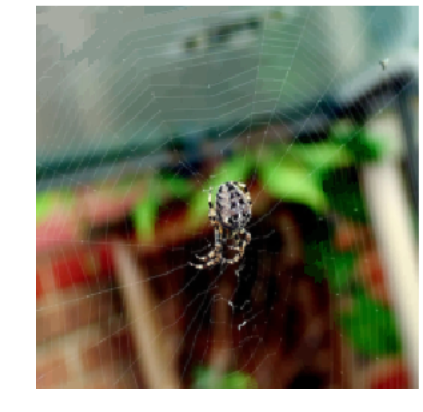} &
    \includegraphics[width=0.16\linewidth]{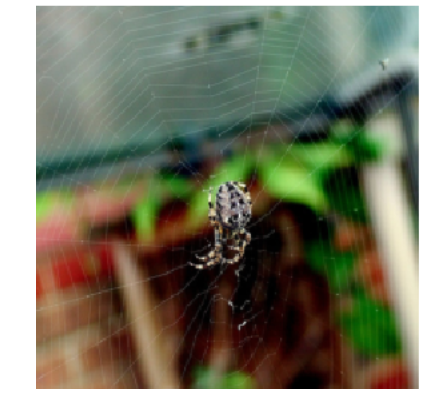} \\
    \includegraphics[width=0.16\linewidth]{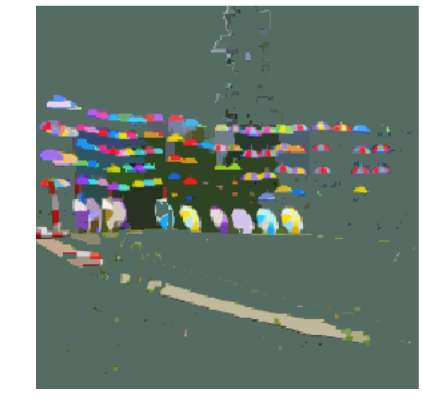} &
    \includegraphics[width=0.16\linewidth]{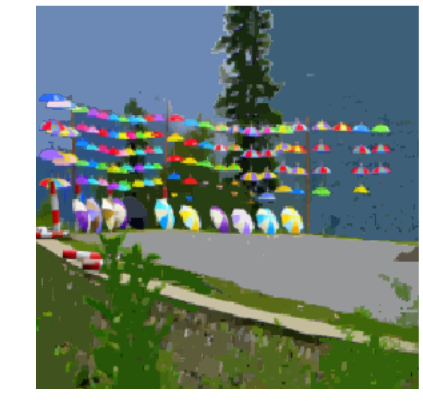} &
    \includegraphics[width=0.16\linewidth]{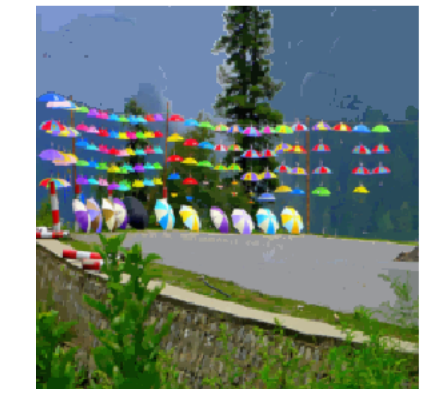} &
    \includegraphics[width=0.16\linewidth]{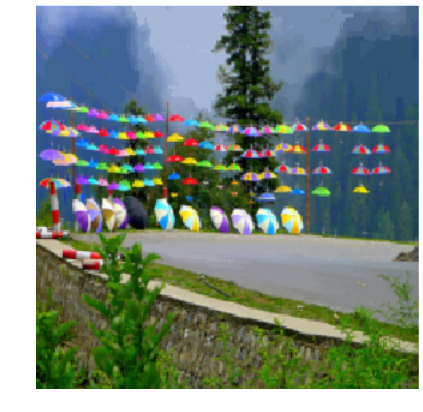} &
    \includegraphics[width=0.16\linewidth]{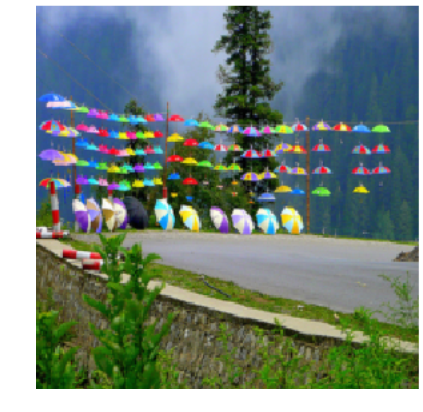} \\
    Sparsity=0.2 & Sparsity=0.4 & Sparsity=0.6 & Sparsity=0.8 & Sparsity=1.0\\
  \end{tabular}
\caption{The image path generated with our instance smoothing algorithm in Eq.~\ref{eq:lbi}. From left to right, the images correspond to sparsity levels of 0.2, 0.4, 0.6, 0.8, and 1.0 (the original image). The 1st to the 6th rows represent a curly-coated retriever; a holster, a dish made of zucchini; a garden spider; a plow; umbrellas.}
  \label{fig.path_shape}
\end{figure}

\begin{figure}[ht!]
  \centering
  \resizebox{\linewidth}{!}{
  \begin{tabular}{ccccc}
    \includegraphics[width=0.16\linewidth]{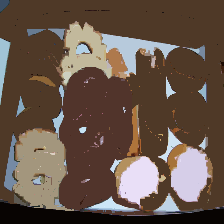} &
    \includegraphics[width=0.16\linewidth]{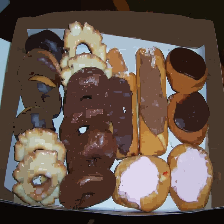} &
    \includegraphics[width=0.16\linewidth]{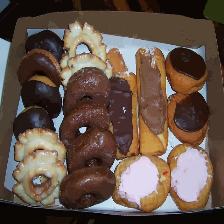} &
    \includegraphics[width=0.16\linewidth]{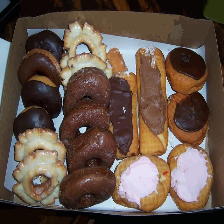} &
    \includegraphics[width=0.16\linewidth]{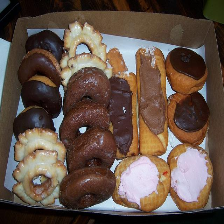} \\
    \includegraphics[width=0.16\linewidth]{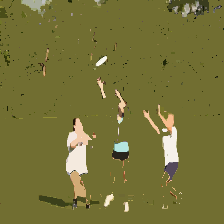} &
    \includegraphics[width=0.16\linewidth]{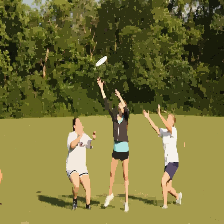} &
    \includegraphics[width=0.16\linewidth]{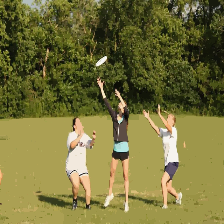} &
    \includegraphics[width=0.16\linewidth]{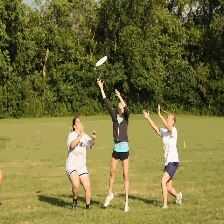} &
    \includegraphics[width=0.16\linewidth]{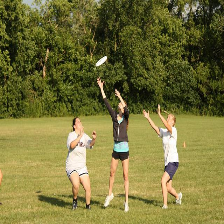} \\
    \includegraphics[width=0.16\linewidth]{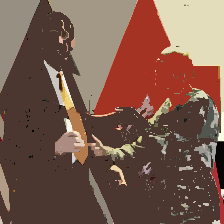} &
    \includegraphics[width=0.16\linewidth]{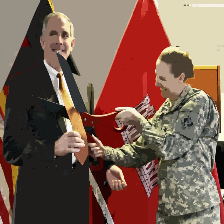} &
    \includegraphics[width=0.16\linewidth]{10.6_coco1.png} &
    \includegraphics[width=0.16\linewidth]{10.8_coco1.png} &
    \includegraphics[width=0.16\linewidth]{coco1.png} \\
   Sparsity=0.2 &  Sparsity=0.4 &  Sparsity=0.6 & Sparsity=0.8 &  Sparsity=1.0 
  \end{tabular}}
  \caption{The image path generated with our instance smoothing algorithm in Eq.~\ref{eq:lbi} for multi-object images (COCO dataset).}
  \label{fig.sol_path_coco}
\end{figure}

\section{More Results of Adversarial Robustness}
\label{sec.more_adv}

We additionally compare our method to some adversarial defense methods, including PNI \citep{he2019parametric}, SAT\citep{huang2020self}, AWP\citep{wu2020adversarial} and adversarial training results from \citep{wang2020high} into comparison with ResNet18 as backbone and $\varepsilon={8}/{255}$.

\begin{table}
\caption{Comparison with other adversarial defence method.}
\label{tab:adv2}
\centering
\resizebox{1.0\linewidth}{!}{
\begin{tabular}{c|ccccccccccc}
\toprule[1pt]
\diagbox{Dataset}{Method} & Adv Training \citep{wang2020high} & SAT\citep{huang2020self} & AWP\citep{wu2020adversarial} & PNI (o)\citep{he2019parametric} & PNI (w) & TV layer (o)\citep{yeh2022total} & TV layer (w) & Fix (o) & Fix (w) & Iterative(o) & Iterative (w) \\
\hline
CIFAR10 &  43.50 & \textbf{55.81} & 55.30 & 41.07 & 51.21 & 43.57 & 53.25 & 37.23 & 51.19 & 35.31 & 44.79\\
CIFAR100 & - & - & \textbf{29.09} & 20.78 & 23.06 & - & - & 5.26 & 23.06 & 6.16 & 27.63 \\
\bottomrule[1pt]
\end{tabular}}
\end{table}

\section{Social Impact}

This work in this paper can make the network more robust to various types of noise. Thus it may improve the safety of some devices like self-driving cars. It has the potential of well protecting the society.

\end{document}